\documentclass{article}

\usepackage[utf8]{inputenc} 
\usepackage[T1]{fontenc}      
\usepackage{url}            
\usepackage{booktabs}       
\usepackage{amsfonts}       
\usepackage{nicefrac}       
\usepackage{microtype}      
\usepackage{lipsum}		
\usepackage{graphicx}
\usepackage{natbib}
\usepackage{doi}

\usepackage{arxiv}
\usepackage{times}
\usepackage{soul}
\usepackage{url}
\usepackage{hyperref}
\usepackage[small]{caption}
\usepackage{graphicx}
\usepackage{amsmath, amssymb}
\usepackage{amsthm}
\usepackage{booktabs}
\usepackage{algorithm}
\usepackage{algorithmic}
\usepackage[switch]{lineno}
\usepackage{bbding}
\usepackage{todonotes}

\newtheorem{definition}{Definition}
\newtheorem{theorem}{Theorem}
\newtheorem{proposition}{Proposition}

\newtheorem{remark}{Remark}

\newcommand{\R}{\mathbb{R}}
\newcommand{\E}{\mathbb{E}}
\newcommand{\indep}{\perp\!\!\!\perp}

\newcommand{\T}{\mathcal{T}}
\newcommand{\G}{\mathcal{G}}
\newcommand{\X}{\mathbf{X}}
\newcommand{\tglob}{\mathbf{T}^{\mathrm{glob}}}
\newcommand{\tlocal}{\mathbf{T}^{\mathrm{loc}}}
\newcommand{\supp}{\mathrm{supp}}

\title{The Failures of Marginal Influence-Based Attribution Methods for Global Time Series Explanations}

\author{ \href{https://orcid.org/0009-0008-5195-090X}{\includegraphics[scale=0.06]{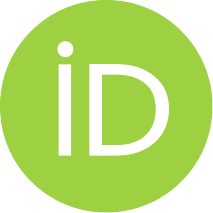}\hspace{1mm}Amadeo Tunyi} \\
	XITASO GmbH \\
	Austraße 35, 86153 Augsburg, Germany \\
	\texttt{amadeo.tunyi@xitaso.com} \\
}

\hypersetup{
pdftitle={A template for the arxiv style},
pdfsubject={q-bio.NC, q-bio.QM},
pdfauthor={David S.~Hippocampus, Elias D.~Striatum},
pdfkeywords={First keyword, Second keyword, More},
}

\begin{document}

\maketitle

\begin{abstract}
Explainability methods for time series models predominantly produce flat attribution 
scores: they quantify the direct influence of a feature at a timestamp by a scalar. We prove that the dominant failure mode of such methods is not the scalar format itself but a fundamental computational mismatch: existing 
methods compute scores via marginal conditioning or off-manifold gradients, both of 
which conflate direct temporal dependencies with mediated ones under autocorrelation. We also define DAG-faithfulness: an explanation is DAG-faithful if the temporal dependency graph it encodes is 
Markov-equivalent to the temporal directed acyclic graph (DAG) implicitly learned by 
the model. Particularly, we observe that standard attribution methods, specifically SHAP, are not DAG-faithful in general, and that recent 
time-series-aware extensions inherit the same computational limitation.
\end{abstract}

\section{Introduction}

The deployment of deep learning on time series data in domains ranging from medical monitoring to financial forecasting to industrial fault detection has created pressing demand for explanations that are both human-comprehensible and technically faithful to the model's reasoning process. The field of explainable AI (XAI) has produced a rich toolkit of attribution methods, yet the dominant paradigm remains \emph{importance scoring}: each input feature at each timestep receives a scalar value indicating its "direct" contribution to a given prediction.

We contend that this paradigm is structurally inadequate for time series models. The reasoning of a temporal model, whether an RNN, a Transformer, or a temporal convolutional network, is inherently relational: what matters is not merely \emph{that} a feature took a particular value, but \emph{how} its value at one time influenced variables at another time, conditioned on the broader temporal context. Most existing importance-scoring methods do not express this; they ignore a high-dimensional relational structure for a one-dimensional summary, potentially discarding the conditional dependency information that constitutes the model's actual reasoning.

This paper makes the following contributions:
\begin{enumerate}
    \item We formalize the computational gap between existing flat attribution scoring methods and the conditional dependency structure of temporal models and on-manifold dependency (Section~\ref{sec:gap}).
    \item We define \emph{DAG-faithfulness} as the appropriate faithfulness criterion, together with admissibility conditions that constrain candidate attribution scoring methods (Section~\ref{sec:framework}).
    \item We show that standard attribution methods fail DAG-faithfulness by construction, that time-series-aware extensions inherit the same failure mode.
    \item We outline an evaluation protocol and identify open research challenges, centrally the construction of admissible and DAG-faithful instantiations (Sections~\ref{sec:evaluation} and~\ref{sec:discussion}).
\end{enumerate}

\section{Background and Related Work}

\subsection{XAI for Time Series}
Existing XAI methods for time series fall into four broad families. \emph{Gradient-based methods}~\cite{simonyan2014deep,sundararajan2017axiomatic} compute the sensitivity of the output with respect to each input dimension, yielding a saliency map of shape $T \times D$. \emph{Perturbation-based methods}~\cite{ribeiro2016why,lundberg2017unified} estimate feature importance by measuring output changes under input occlusion, again collapsing to scalar attributions per feature--time pair. \emph{Attention-based methods}~\cite{vaswani2017attention} expose the model's internal attention weights, which form a matrix over timesteps; whether these weights constitute faithful explanations remains contested~\cite{jain2019attention,wiegreffe2019attention,abnar2020quantifying}.

A fourth, more recent, family of \emph{temporally-aware attribution methods} explicitly accounts for temporal dependence in the computation of importance. FIT~\cite{tonekaboni2020what} quantifies importance via KL-divergence of the predictive distribution under feature occlusion. Dynamask~\cite{crabbe2021explaining} learns near-binary perturbation masks over the input sequence. WinIT~\cite{leung2023temporal} extends feature-removal scores over sliding windows to capture delayed effects. TimeSHAP~\cite{bento2021timeshap}  and ShapTime~\cite{ZhangSQLMP23} adapt SHAP to sequence models with event-level and feature-level Shapley values. These methods move beyond generic saliency by modeling the temporal structure of the input, but as we show in Section~\ref{sec:gap}, some of these methods fail under auto-correlation: they cannot tell between mediated effect on $X_j^{(t)}$ of $X_i^{(t-k)}$ \emph{via} $X_l^{(t-k')}$, $0 < k' \le k$, $i, j, l \in [D]$ and direct effect.

\subsection{Causal Structure in Time Series}

Granger causality~\cite{granger1969investigating} operationalizes temporal influence as predictive improvement: $X$ Granger-causes $Y$ if past values of $X$ improve the forecast of $Y$ beyond past values of $Y$ alone. Structural causal models~\cite{peters2017elements} and dynamic Bayesian networks~\cite{murphy2002dynamic} provide richer frameworks for representing conditional independence in temporal data. Modern causal discovery methods for time series, including constraint-based approaches such as PCMCI+~\cite{runge2020discovering} and score-based approaches such as DYNOTEARS~\cite{pamfil2020dynotears} and Rhino~\cite{gong2023rhino}, operate on the raw data to recover its underlying dependency structure. Our framework connects XAI faithfulness to these causal structures but reverses the target: rather than discovering dependencies in the data, we seek to recover the dependencies that a trained model has learned, grounding explanation quality in the model's learned probabilistic geometry rather than its output value alone.

\subsection{Faithfulness in XAI}

Existing faithfulness criteria are predominantly \emph{output-centric}: an explanation is faithful if perturbing highly-attributed features changes the output more than perturbing low-attributed features~\cite{samek2016evaluating,alvarez2018robustness}. This criterion is necessary but insufficient; a method can pass it while misrepresenting the model's internal reasoning structure~\cite{kindermans2019reliability,jacovi2020towards}. A parallel line of work on mechanistic and causal-abstraction-based faithfulness~\cite{geiger2021causal} moves toward structure-centric criteria by asking whether an explanation preserves the model's internal causal structure under interventions. We contribute to this shift by proposing a faithfulness criterion specific to temporal models: the explanation must recover the model's conditional independence structure, not merely correlate with its outputs.

\section{The Computational Gap}
\label{sec:gap}
This section makes the central claim precise. We first fix notation, then define the class of methods we target (flat attribution) and the structure they ought to recover (the model dependency graph). We then isolate a property that any faithful explanation (DAG faithfulnes) must have and prove that no flat method can hold it. A proof argument is presented in the appendix.

Denote $[D] = \{1, \dots, D\}$ and $[T] = \{1, \dots, T\}$. Let $\X = \{X_i^{(t)}\}_{i\in [D]\; t\in [T]}$ denote a multivariate time series with $D$ variables and $T$ timesteps. Let $f: \mathbb R^{T \times D} \to \mathbb R^{W \times D}$ be a trained time series model with $W$ the forecast window size, that takes such a window as input and produces a $W \times D$-dimensional forecast at the next $W$ time steps. Throughout, $p(\mathbf{x})$ denotes the \emph{reference distribution} over inputs against which the explanation is defined: a probability measure on $\R^{T\times D}$ with respect to which all expectations, conditional independences, and support statements are taken. In practice, $p$ is the deployment distribution of $f$, approximated empirically by the evaluation set on which the global explanation is computed.

The methods we critique share one feature: each score is a function of a single input coordinate considered in isolation, whether through marginal conditioning or through a gradient. We name this class precisely.

\begin{definition}[Flat Attribution]
A flat attribution method produces a map $\phi: \mathbb R^{T \times D} \to \mathbb R^{T \times D}$, assigning a scalar importance score $\phi_i^{(t)}(\mathbf{x})$ to each variable-time pair $(X_i, t)$ for a given input $\mathbf{x}$ via marginal or gradient based methods.
\end{definition}

The label ``flat'' refers not to the scalar output but to the computation: the score for $(X_i,t)$ never conditions on the value of the other inputs that may mediate its effect.
Contrary to classic XAI or even XAI for time series classification that use flat attribution scoring, XAI for forecasting involves not only deriving feature relevance but also understanding it with respect to the model's time-dependent forecasting dynamics; i.e., an appropriate temporal-aware XAI method should be able to properly portray the data dynamics as learned by the model in a faithful way. The most interpretable way of portraying such a dynamic is via a Directed Acyclic Graph (DAG), namely a \emph{Model Dependency Graph} that encodes forecasting dependencies via edges between vertices which we now define.

\begin{definition}[Model Dependency Graph (MDDAG)]
\label{def:model-dag}
Let $f: \mathbb{R}^{T \times D} \to \mathbb{R}^{W \times D}$ be a trained 
time series model and let $p(\mathbf{x})$ be the reference distribution. 
The \emph{model dependency graph (MDDAG)} $\mathcal{G}_f$ is a directed acyclic 
graph with node set
\[
V = \bigl\{ X_i^{(t)} : i \in [D],\ t \in [T+W] \bigr\},
\]
where the $D \cdot T$ nodes with $t \in [T]$ are input positions and the 
$D \cdot W$ nodes at $t \in \{T+1, \ldots, T+W\}$ are forecast targets 
(sinks). A directed edge $X_i^{(t_1)} \to X_j^{(T+w)}$, $j \in [D]$ with $t_1 \in [T]$ 
and $w \in [W]$ belongs to $\mathcal{G}_f$ if and only if
\[
X_i^{(t_1)} \not\perp\!\!\!\perp f(\mathbf{X})_{j,w} 
\;\Big|\; 
\mathbf{X} \setminus \{X_i^{(t_1)}\},\ 
f(\mathbf{X})_{-j,w},\ 
f(\mathbf{X})_{\cdot,\neq w}
\]
under $p(\mathbf{x})$, where $f(\mathbf{X})_{j,w}$ denotes the model's 
forecast for variable $j$ at horizon $w$, $f(\mathbf{X})_{-j,w}$ denotes 
the forecasts of all other variables at the same horizon $w$, and 
$f(\mathbf{X})_{\cdot,\neq w}$ denotes the full joint forecast at all 
other horizons $w' \neq w$. Acyclicity holds by construction: edges run 
only from input positions $t \in [T]$ to forecast targets 
$t \in \{T+1,\ldots,T+W\}$. An illustration is presented in Figure~\ref{fig:MDDAG}.
\end{definition}

\begin{figure}[ht]
    \centering
    \includegraphics[scale=1.0]{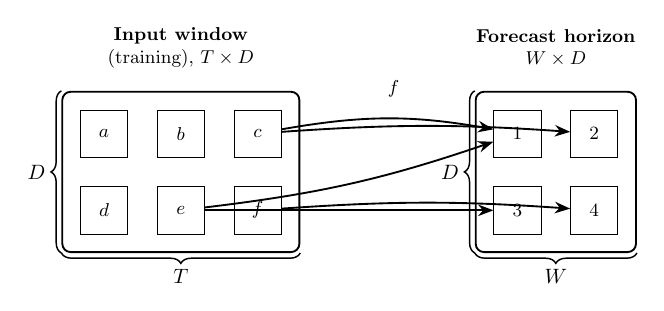}
    \caption{\textbf{The model dependency graph $G_f$.} A trained model
$f:\mathbb{R}^{T\times D}\!\to\!\mathbb{R}^{W\times D}$ maps an input (training)
window of $T$ timesteps over $D$ variables to a $W\times D$ forecast horizon. An
edge $X_i^{(t_1)}\!\to\!X_j^{(T+w)}$ is present iff input cell $i$ at time $t_1$
has a \emph{direct} influence on forecast cell $j$ that is not mediated by any
other input. Here the model has learned the dependencies
$c\!\to\!1$, $c\!\to\!2$, $e\!\to\!1$, $e\!\to\!3$, $f\!\to\!4$; the remaining
inputs ($a,b,d$) are edge-free, exerting no direct influence on the forecast.}

    \label{fig:MDDAG}
\end{figure}
Faithfulness now becomes the guarantee that the XAI method properly learns this model dependency graph. We introduce a novel metric, \emph{DAG-faithfulness}, to quantify this guarantee, using notions from graph theory and causal inference. DAG-faithfulness examines if the dependencies learned by the XAI match the edges in the MDDAG.

\begin{definition}[DAG-Faithfulness]
\label{def:faithfulness}
An explanation method is \emph{DAG-faithful} if the dependency graph $\G_{\T}$ it induces encodes the exact same conditional independencies (Markov Equivalence) as the MDDAG $\G_f$. That is, for all input--forecast pairs $(X_i^{(t_1)}, X_j^{(T+w)})$ with $t_1 \in [T]$, $w \in [W]$:
\[
    X_i^{(t_1)} \indep_{\G_{\T}} X_j^{(T+w)} \mid \mathrm{rest}
    \iff
    X_i^{(t_1)} \indep_{\G_f} X_j^{(T+w)} \mid \mathrm{rest}
\]
where $\indep_{\G}$ denotes independence in graph $\G$.
\end{definition}

\begin{figure}
    \centering
    \includegraphics[scale=0.8]{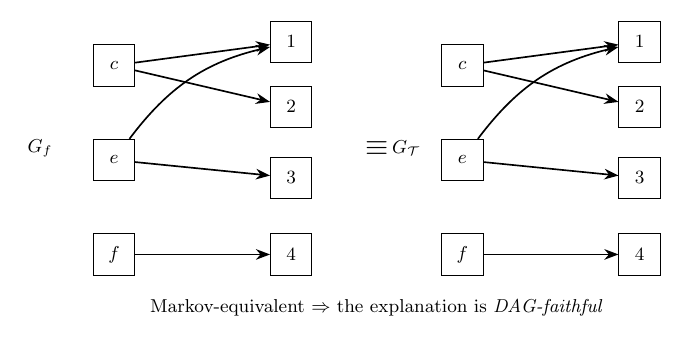}
    \caption{\textbf{DAG-faithfulness.} A global explanation is DAG-faithful when the
graph $G_{\mathcal{T}}$ induced by its transition tensor is Markov-equivalent to
the model dependency graph $G_f$. Here the two coincide: every source–target
dependency the model has learned ($c\!\to\!1$, $c\!\to\!2$, $e\!\to\!1$,
$e\!\to\!3$, $f\!\to\!4$) is recovered, and the edge-free inputs are correctly
left unconnected, so the explanation neither hallucinates nor misses an edge.}
    \label{fig:dag_faithfulness}
\end{figure}
\begin{remark}
DAG-faithfulness strictly implies output-centric faithfulness. If an explanation recovers the correct conditional independence structure, then perturbing edges in $\G_{\T}$ will correspond to perturbing the model's internal dependencies, which will necessarily affect outputs. The converse does not hold.
\end{remark}
Requiring Markov equivalence is a rather strict constraint\footnote{Requiring Markov equivalence rather than exact graph recovery is deliberate and principled: it is the strongest structural claim that can be made from observational data alone, as multiple DAGs with different edge orientations may be indistinguishable without interventional data~\cite{pearl2009causality}.}; however, DAG faithfulness can be reduced to two fundamental admissibility conditions
\begin{itemize}
    \item \textbf{(C1) On-manifold evaluation.} The attribution functional depends on $f$ only through its behavior on the data distribution support $\supp(p)$. Two models that agree on $\supp(p)$ must receive identical attributions; in particular, the attribution must not depend on off-manifold reparametrizations of $f$.
    \item \textbf{(C2) Conditional separation.} The attribution to input $X_i^{(t_1)}$ vanishes whenever $X_i^{(t_1)}$ is conditionally independent of $f(\X)$ given the remaining inputs under $p(\mathbf{x})$.
\end{itemize}
In what follows, we show that no flat attribution method, as defined, can hold either C1 or C2, or both. First, we make the following claim
\paragraph{Claim (thresholding commits the explanation to a graph).}
Any flat attribution method $\phi$ assigns a single global score $\Phi^f_{i,\,t_1}$ to each input variable-time pair $(X_i, t_1)$ with $t_1 \in [T]$. Thresholding this score at $\tau \geq 0$ induces a graph $\G_\phi^\tau$ in which $(X_i, t_1)$ is connected to the forecast if and only if $\Phi^f_{i,\,t_1} > \tau$. A faithful explanation, therefore, requires:
\[
\Phi^f_{i,\,t_1} > \tau \iff \exists\, j \in [D] : X_i^{(t_1)} \to X_j^{(T+1)} \in \G_f,
\]
i.e., the attribution exceeds threshold iff the input has at least one direct edge to some forecast target (the specific $j$ is retrievable by applying the attribution method per output coordinate). In the single-variable setting ($D = 1$), the existential collapses to $X^{(t_1)} \to X^{(T+1)} \in \G_f$, and we write $\Phi^f_{t_1}$ for the scalar attribution (formally defined in \eqref{eq:marginal} below).
\subsection{Case Studies}
Consider two scenarios~\footnote{Codes available at \href{https://github.com/AmTuTi1999/FMIBAMGTSE/}{https://github.com/AmTuTi1999/FMIBAMGTSE/}}
\paragraph{Scenario I:} Let $D = 1$ (single variable, so we drop the $i$ index) and $T = 2$, with input window $\{X^{(1)}, X^{(2)}\}$ and forecast target at position $T+1 = 3$. Suppose the data-generating process satisfies
\begin{equation}
    X^{(2)} = \delta\,X^{(1)} + \varepsilon, \quad
    \varepsilon \sim \mathcal{N}(0,1), \;
    \varepsilon \indep X^{(1)}, \; \delta \neq 0,
    \label{eq:dgp}
\end{equation}
inducing temporal autocorrelation $\mathrm{Cov}(X^{(2)}, X^{(1)}) \neq 0$. Define
\begin{align}
    f_1(\mathbf{x}) &= \beta\,x^{(2)}, \label{eq:f1} \\
    f_2(\mathbf{x}) &= \beta\,x^{(2)} + \gamma\,x^{(1)}, \label{eq:f2}
\end{align}
with $\beta, \gamma \neq 0$. Applying Definition~\ref{def:model-dag} directly: conditioning on $X^{(2)}$, $f_1(\X) = \beta X^{(2)}$ is deterministic in $X^{(2)}$ alone, so $X^{(1)} \indep f_1(\X) \mid X^{(2)}$ and $X^{(1)} \to X^{(3)} \notin \G_{f_1}$. For $f_2$, the same conditioning leaves a $\gamma X^{(1)}$ term, so $X^{(1)} \to X^{(3)} \in \G_{f_2}$. Hence $\G_{f_1} \not\cong \G_{f_2}$, and (C2) requires zero attribution to $X^{(1)}$ under $f_1$ and nonzero under $f_2$.
\begin{figure}[ht]
    \centering
    \includegraphics[scale=0.8]{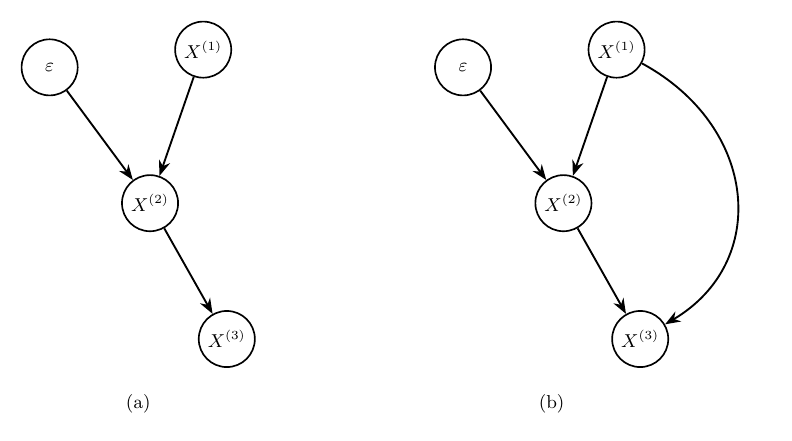}
    \caption{\textbf{SCENARIO I: data-generating structure and the two
models.} Under the process $X^{(2)}=\delta X^{(1)}+\varepsilon$ with
$\varepsilon\!\perp\!\!\!\perp\!X^{(1)}$, the input $X^{(1)}$ influences the
forecast $X^{(3)}$ only through the mediator $X^{(2)}$ (autocorrelation).
\textbf{(a)} $G_{f_1}$ for $f_1=\beta X^{(2)}$: the only edge into the forecast
is $X^{(2)}\!\to\!X^{(3)}$, so $X^{(1)}$ has no direct effect. \textbf{(b)}
$G_{f_2}$ for $f_2=\beta X^{(2)}+\gamma X^{(1)}$: a genuine direct edge
$X^{(1)}\!\to\!X^{(3)}$ is present. The two graphs differ in exactly one edge,
yet a flat marginal score assigns $X^{(1)}$ the same value under both, unable to
separate them.}
\label{fig:running-example-models}
\end{figure}

\paragraph{Scenario II:}  Let $D = 1$, $T = 2$ with input window $\{X^{(1)}, X^{(2)}\}$ and forecast target $X^{(3)}$, and assume the deterministic data-generating process
\begin{equation}
X^{(2)} = \delta\,X^{(1)}, \qquad \delta \neq 0,
\label{eq:dgp-det}
\end{equation}
so $\supp(p)$ is the line $\{(x^{(1)}, x^{(2)}) : x^{(2)} = \delta x^{(1)}\}$ in $\R^2$. Define
\begin{align*}
f_1(\mathbf{x}) &= \beta\,x^{(2)}, \\
f_3(\mathbf{x}) &= \beta\,x^{(2)} + \gamma\,\bigl(x^{(2)} - \delta\,x^{(1)}\bigr).
\end{align*}
On $\supp(p)$, $x^{(2)} - \delta x^{(1)} = 0$ identically, so $f_3 = f_1$ pointwise on $\supp(p)$. Hence $\G_{f_3} = \G_{f_1}$, and the edge $X^{(1)} \to X^{(3)}$ is absent from both. (C2) is satisfied: a method conditioning on the actual conditional distribution would assign zero to $X^{(1)}$ under both models.
\begin{figure}[ht]
    \centering
    \includegraphics[scale=0.9]{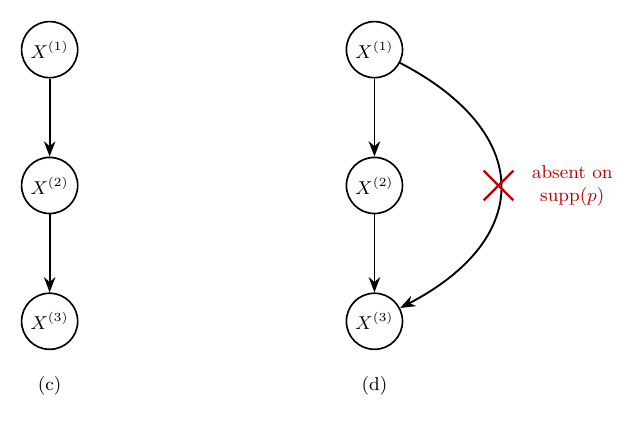}
\caption{\textbf{SCENARIO II.} \textbf{(c)} The true
dependency structure shared by $f_1$ and $f_3$: the chain
$X^{(1)}\!\to\!X^{(2)}\!\to\!X^{(3)}$, with no direct edge
$X^{(1)}\!\to\!X^{(3)}$. \textbf{(d)} On the deterministic manifold
$X^{(2)}=\delta X^{(1)}$, the model $f_3$ agrees with $f_1$ everywhere on
$\mathrm{supp}(p)$, so the direct edge $X^{(1)}\!\to\!X^{(3)}$ is \emph{absent on
$\mathrm{supp}(p)$} (crossed). A gradient method nonetheless reports
$\partial f_3/\partial x^{(1)}=-\gamma\delta\neq0$, registering off-manifold
sensitivity and thus violating on-manifold evaluation (C1).}
\label{fig:running-example-gradient}
\end{figure}
\subsubsection{Case Study A: Marginal Influence Based Methods}
For our investigation, we consider the classic marginal influence-based methods KernelSHAP (KS), TimeSHAP (TS), ShapTime (ST), and TS-MULE (TM). We proceed in two ways, analytically and empirically. For the analytical proof, see Appendix. We make the following conjecture,
\begin{proposition}
    These marginal influence-based methods falter in scenario I and/or II, thus failing C1 and/or C2.
\end{proposition}
\noindent The interpretation of this proposition:
\begin{itemize}
    \item Conditioning on a single variable fails to capture autocorrelation,
    \item Baseline methods like SHAP are likely to lay mass on off-manifold data.
\end{itemize}
To prove our proposition, we present empirical evidence from the experiment.

\begin{table}[h]
\centering
\begin{tabular}{llrcrcrcc}
\toprule
Method  & $\Phi^{f_1}_1$ & $f_1$ ok & $\Phi^{f_2}_1$ & $f_2$ ok & $\Phi^{f_3}_1$ & $f_3$ ok \\
\midrule
KS& 0.00 & \checkmark & 0.42 & \checkmark & 0.20 & $\times$ \\
TS& 0.00 & \checkmark & 0.22 & \checkmark & 0.11 & $\times$ \\
ST& 1.27 & $\times$ & 1.60 & \checkmark & 1.80 & $\times$ \\
TM& 0.01 & $\times$ & 0.40 & \checkmark & 0.20 & $\times$ \\
\bottomrule
\end{tabular}
\caption{\textbf{Attribution to $X^{(1)}$ in Scenario I and II.} Each method's
global score $\Phi^f_1$ for input $X^{(1)}$ under the three models
$f_1,f_2,f_3$. The correct verdict is $\Phi_1=0$ under $f_1$ and $f_3$ (where
$X^{(1)}$ has \emph{no} direct edge to the forecast) and $\Phi_1>0$ under $f_2$
(where it does); \checkmark/$\times$ marks whether the method agrees. KS
(KernelSHAP) and TS (TimeSHAP) handle $f_1$ and $f_2$ correctly but assign a
spurious nonzero score under $f_3$, where $f_3\!\equiv\!f_1$ on
$\mathrm{supp}(p)$, a violation of on-manifold evaluation (C1). ST (ShapTime)
and TM (TS-MULE) fail already on $f_1$, attributing influence to the
fully-mediated $X^{(1)}$, a violation of conditional separation (C2). No method
is correct across all three models.}
\label{tab:results1}
\end{table}

The results in Table~\ref{tab:results1} prove our proposition. We observe that even when the SHAP-based methods TS and KS correctly identify spurious relationships caused by autocorrelation, they fail under on-manifold constraints caused by faulty baselines used in their marginalization strategies. And ShapTime and TS-MULE fail in all cases where autocorrelation is present.
\subsection{Case Study: Gradient-Based Methods}
We consider in this case Saliency (SA), Integrated Gradients (IG), its extension Temporal Integrated Gradients (TIG), and SmoothGrad (SG). 

\begin{proposition}
\label{prop:scenarioii}
    Gradient-based methods fail in scenario II, thus failing C1.
\end{proposition}
\noindent The interpretation:
\begin{itemize}
    \item Gradient computation causes a perturbation that makes explanations learn on off-manifold data.
\end{itemize}

\begin{table}[h]
\centering
\begin{tabular}{llrcrcrcc}
\toprule
Method  & $\Phi^{f_1}_1$ & $f_1$ ok & $\Phi^{f_2}_1$ & $f_2$ ok & $\Phi^{f_3}_1$ & $f_3$ ok \\
\midrule
SA  & 0.00 & \checkmark & 0.50 & \checkmark & 0.25 & $\times$ \\
IG  & 0.00 & \checkmark & 0.40 & \checkmark & 0.20 & $\times$ \\
TIG & 0.00 & \checkmark & 0.40 & \checkmark & 0.20 & $\times$ \\
SG  & 0.00 & \checkmark & 0.50 & \checkmark & 0.25 & $\times$ \\
\bottomrule
\end{tabular}
\caption{\textbf{Gradient-based attribution.}
Global score $\Phi^f_1$ for input $X^{(1)}$ under the three models; the correct
verdict is $\Phi_1=0$ under $f_1,f_3$ (no direct edge $X^{(1)}\!\to\!X^{(3)}$)
and $\Phi_1>0$ under $f_2$ (direct edge present). SA (Saliency), IG (Integrated
Gradients), TIG (Temporal IG), and SG (SmoothGrad) all separate $f_1$ from
$f_2$ correctly, but every method assigns a spurious nonzero score under $f_3$:
since $f_3\!\equiv\!f_1$ on $\mathrm{supp}(p)$ yet
$\partial f_3/\partial x^{(1)}=-\gamma\delta\neq0$, the gradient registers
off-manifold sensitivity, violating on-manifold evaluation (C1). The uniform
failure on $f_3$ is the gradient-family signature, complementing the marginal
methods' (C2) failures in Table~\ref{tab:results1}.}
\label{tab:result2}
\end{table}

Table~\ref{tab:result2} proves Proposition~\ref{prop:scenarioii}. Indeed, in scenario II, we expect variable $X_1$ to be attributed $0$ relevance towards prediction in $f_3$ on $supp(p)$, but for all gradient-based methods, we consistently get non-zero attributions for $X_1$ under $f_3$, indicating that they break on-manifold restrictions.

\section{Admissible Instantiations of Temporal-aware Explanations}
\label{sec:framework}
\subsection{Existing Methods}
Among existing methods, the temporally-aware perturbation approaches, namely FIT,
WinIT and Dynamask are the closest to admissibility, i.e., their underlying structures fulfill C1 and C2: 

\paragraph{FIT} scores each observation by its contribution to the shift in the model's predictive distribution over time, measured as a KL divergence between the prediction given all observed features and the prediction under a counterfactual in which the target feature is treated as unobserved, that feature being replaced by a draw from a learned conditional generator $p(x_{\text{unobs}}\mid x_{\text{obs}})$
 rather than a fixed value.
 
\paragraph{WinIT} extends FIT in two ways: it measures a feature's importance on later predictions by aggregating over a temporal window (capturing delayed effects), and it accounts for dependencies between features across time. Like FIT, it removes a feature by sampling an in-distribution replacement from a conditional generator and scores the resulting predictive-distribution change over the window.

\paragraph{Dynamask} learns a near-binary mask $m\in[0,1]^{T\times D}$ by optimizing a perturbation objective with a sparsity/area regularizer; masked entries are replaced not by a constant baseline but by a temporally coherent perturbation operator (e.g., a moving-average / blurred version of the series), so the perturbed input stays close to the local temporal structure.

 Each already incorporates,
to a degree, the two ingredients our conditions require, namely conditional,
in-distribution replacement of removed features (via a learned conditional
generator for FIT and WinIT (WI), or a temporally coherent perturbation operator for
Dynamask (DM)). Empirically they remain only partially reliable, their edge-recovery
accuracy still falls short of what DAG-faithfulness demands, but their structure
is the right one, and improving the fidelity of the conditional sampler and the
choice of conditioning set is a promising route toward an admissible standard for
explainability in time-series forecasting.
\begin{table}[h]
    \centering
\begin{tabular}{llrcrcrcc}
\toprule
Method  & $\Phi^{f_1}_1$ & $f_1$ ok & $\Phi^{f_2}_1$ & $f_2$ ok & $\Phi^{f_3}_1$ & $f_3$ ok \\
\midrule

WI  & 0.00 & \checkmark & 0.00 & $\times$ & 0.00 & \checkmark \\
DM  & 0.00 & \checkmark & 0.01 & \checkmark & 0.00 & \checkmark \\
FIT  & 0.00 & \checkmark & 0.00 & $\times$ & 0.00 & \checkmark \\
\bottomrule
\end{tabular}
\caption{\textbf{Temporally-aware perturbation methods.}
Global score $\Phi^f_1$ for $X^{(1)}$ under the three models; correct is
$\Phi_1=0$ for $f_1,f_3$ and $\Phi_1>0$ for $f_2$. Unlike the marginal and
gradient families (Tables~\ref{tab:results1}--\ref{tab:result2}), WI (WinIT),
DM (Dynamask) and FIT do not hallucinate the mediated edge under $f_1$ or the
off-manifold edge under $f_3$: by replacing removed features with conditional,
in-distribution samples they partially respect both admissibility conditions.
Their weakness is the opposite one, low sensitivity to the \emph{genuine}
direct edge: WI and FIT score $0$ under $f_2$ (a missed true edge, a false
negative), and DM detects it only weakly ($0.01$). These methods thus err toward
under-attribution rather than the spurious over-attribution of the other
families, which is why we regard them as the most promising starting point for an admissible instantiation.}
    \label{tab:result3}
\end{table}

\subsection{The Way Forward}

The failure identified in Section~\ref{sec:gap} is, at its root,
\emph{representational}: a single scalar per feature cannot record \emph{which}
other variables a dependency is conditioned on. The minimal object that can is a
tensor indexed by source variable, target variable, and lag.

\begin{definition}[Transition Tensor]
\label{def:transition-tensor}
A \emph{transition tensor} for a model $f$ at input $\mathbf{x}$ is a three-dimensional array
\[
\tlocal(\mathbf{x}) \in \R^{D \times D \times T},
\]
whose entry $\tlocal_{ij}{}^{(t_1)}(\mathbf{x})$ quantifies the influence of $X_i^{(t_1)}$ on $f(\mathbf{x})_j$ at the forecast target, for input position $t_1 \in [T]$. The forecast time is fixed at $T+1$, so a single time axis (the source position $t_1$) suffices. When $W >1$, we have $W$ such containers, that is, 
\[
\tlocal(\mathbf{x}) \in \R^{D \times D \times T \times W},
\]
and entry $\tlocal_{ij}{}^{(t_1), w}(\mathbf{x})$ quantifies the influence of $X_i^{(t_1)}$ on forecast at $f(\mathbf{x})_j$ at forecast time $T+w$, $w \in [W]$. A \emph{Global Transition Tensor} summarizes the model's behavior across the reference
distribution, so we aggregate the per-instance tensors.
\[
    \tglob = \E_{\mathbf{x} \sim p(\mathbf{x})} \left[ \tlocal(\mathbf{x}) \right]
\]
where $p(\mathbf{x})$ is the reference distribution. In practice, this is estimated as
\[
    \hat{\tglob} = \frac{1}{N} \sum_{n=1}^{N} \tlocal(\mathbf{x}^{(n)})
\]
over an evaluation set of $N$ samples drawn from $p$.
The DAG structure of the global explanation is read from the \emph{support} of $\tglob$, not its magnitude: an edge $X_i^{(t_1)} \to X_j^{(T+1)}$ is present in the induced graph $\G_{\T}$ if and only if $|\tglob_{ij}{}^{(t_1)}| > \epsilon$ for some threshold $\epsilon > 0$. This cleanly separates \emph{structural} information (which dependencies exist, global and discrete) from \emph{strength} information (how strong each dependency is, which may vary per instance). Figure~\ref{fig:transition-tensor} summarizes the pipeline from local tensors to the induced graph.

\end{definition}

The tensor is only a \emph{structural container}: its shape provides the minimal
capacity needed to encode pairwise temporal dependencies, but the choice of
influence functional $\tlocal_{ij}{}^{(t_1)}(\cdot)$ determines whether a given
instantiation is DAG-faithful (Definition~\ref{def:faithfulness}).

\begin{definition}[Admissible Instantiation]
\label{def:admissible}
An influence functional $\phi$ is \emph{admissible} if its resulting transition
tensor satisfies two conditions:
\begin{enumerate}
    \item \textbf{On-manifold evaluation:} $\phi$ depends on $f$ only through its
    behavior on $\supp(p)$, so that models agreeing on $\supp(p)$ receive
    identical tensors.
    \item \textbf{Conditional separation:} the entry $\tglob_{ij}{}^{(t_1)}$
    vanishes whenever $X_i^{(t_1)} \indep f(\X)_j \mid \X \setminus \{X_i^{(t_1)}\}$
    under $p(\mathbf{x})$.
\end{enumerate}
\end{definition}

\begin{figure}[ht]
  \centering
  \includegraphics[width=\columnwidth]{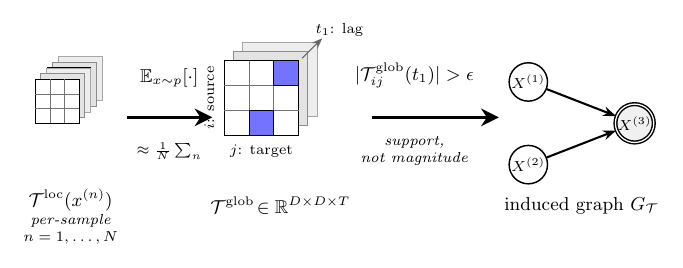}
  \caption{\textbf{From local influence to a global dependency graph.}
  Instance-level transition tensors $\tlocal(\mathbf{x}^{(n)})$
  (Definition~\ref{def:transition-tensor}), indexed by source variable $i$,
  target variable $j$, and source lag $t_1$, are averaged over the reference
  distribution into the global tensor $\tglob$. The induced graph $\G_{\T}$
  (Definition~\ref{def:faithfulness}) is read from the \emph{support} of
  $\tglob$: an edge is present iff $|\tglob_{ij}{}^{(t_1)}|>\epsilon$, separating
  which dependencies exist from how strong each one is.}
  \label{fig:transition-tensor}
\end{figure}

This motivates the position we advocate: \emph{a global explanation method for
time-series forecasting is acceptable only if it is admissible in the sense of
Definition~\ref{def:admissible} and expresses its output as a transition tensor.}
The two conditions are not meant to be verified as external metrics after the
fact, that is intractable in general on real world data, but to be built into the influence
functional \emph{by construction}.

\paragraph{Recommendation.}
A concrete admissible functional must combine two ingredients, one per condition:
\begin{itemize}
    \item \textbf{(C1) On-manifold evaluation}, approached through manifold-aware
    Jacobian entries (directional derivatives), integrated gradients along the
    temporal axis, or attention-flow measures~\cite{abnar2020quantifying};
    \item \textbf{(C2) Conditional separation}, approached through
    conditional-distribution-shift measures such as those underlying FIT and
    WinIT, but conditioning on the \emph{full} set of remaining inputs rather
    than a single variable.
\end{itemize}

The main obstacle lies in (C2): the conditioning set is of order
$\mathcal{O}(DT)$. This can be reduced by replacing the full complement with an
estimated Markov blanket of the source variable, by discretizing the conditioning
variables into bins, or by running the tests in the state space of a lower-order
surrogate that approximates $f$.

\subsection{Evaluation Protocol}
\label{sec:evaluation}

So far DAG-faithfulness has been an abstract criterion. We now make it
\emph{observable}, a quantity that can be measured for an existing method or used
to test a new one. The key is that the global criterion decomposes into
independent per-edge tests.

\begin{proposition}[Edge-local characterization of DAG-faithfulness]
\label{prop:evaluation}
DAG-faithfulness (Definition~\ref{def:faithfulness}) admits an edge-local
characterization that does not require recovering $\G_f$ in advance. For each
candidate edge, the edge condition of Definition~\ref{def:model-dag} reduces to a
conditional-independence test between the source input and the model's forecast,
which can be estimated from model queries and held-out data via kernel-based or
neural conditional-independence estimators.
\end{proposition}

The consequence is practical: since each edge is an independent
conditional-independence test, an explanation's induced graph $\G_{\T}$ can be
scored against $\G_f$ edge by edge, without ever materializing $\G_f$ in full. We
outline the resulting protocol in three settings of decreasing access to ground
truth.

\paragraph{Synthetic benchmarks.}
Generate data from a known dynamic Bayesian network with a fixed ground-truth DAG
$\G^*$, and train a temporal model $f$ on it. Under our model-centric framing the
target of recovery is $\G_f$ (what the model learned), not $\G^*$ (what generated
the data). When $f$ is analytically characterizable e.g.\ a linear VAR with
known coefficients $\G_f$ is known by construction, and the
\emph{structural Hamming distance} between $\G_{\T}$ and $\G_f$ is a direct scalar
measure of faithfulness. When $\G_f$ is not analytically accessible, $\G^*$ serves
as a proxy, with the caveat that recovery error then conflates model-learning
error with explanation error.

\paragraph{Flat-method baselines.}
On the same benchmark, any flat attribution method is evaluated by thresholding
its scores into a binary edge mask and comparing that mask to the skeleton of
$\G_f$. This places flat and tensor-based methods on a common footing and makes
the representational gap of Section~\ref{sec:gap} concrete and measurable.

\paragraph{Real-data evaluation.}
When $\G^*$ is unknown, faithfulness is assessed indirectly through
\emph{interventional consistency}: if the explanation marks an edge
$X_i^{(t_1)} \to X_j^{(T+1)}$ as absent, then intervening on $X_i^{(t_1)}$ should
leave $f$'s forecast for $X_j$ unchanged once the rest of the window is held
fixed. This links the criterion to model debugging and counterfactual
evaluation~\cite{goyal2019counterfactual}.

\section{Analysis of Standard Methods}
\label{sec:analysis}

\subsection{Standard Methods Fail DAG-Faithfulness}

\begin{theorem}[SHAP is not DAG-faithful in general]
\label{thm:shap}
There exists a model $f$ and reference distribution $p(\mathbf{x})$ such that the global SHAP explanation is not DAG-faithful.
\end{theorem}

\begin{proof}[Proof sketch]
Consider a linear model with $D = 2$, $T = 2$ implementing the chain $X_1^{(1)} \to X_1^{(2)} \to Y^{(3)}$, where $Y^{(3)} = \alpha X_1^{(2)}$ and the data-generating process satisfies $X_1^{(2)} = \delta X_1^{(1)} + \varepsilon$ with $\varepsilon$ independent of $X_1^{(1)}$. The model dependency graph $\G_f$ contains the direct edge $X_1^{(2)} \to Y^{(3)}$ but not $X_1^{(1)} \to Y^{(3)}$: conditional on $X_1^{(2)}$, $X_1^{(1)}$ is independent of $Y^{(3)}$ under $p$.

SHAP values (either marginal or conditional variant) assign a nonzero Shapley value to $X_1^{(1)}$ whenever its coalition contributions shift the conditional expectation of $f$, which they do via the autocorrelation structure above. When the induced graph $\G_{\T^{\mathrm{SHAP}}}$ is defined by thresholding Shapley values, it contains the spurious edge $X_1^{(1)} \to Y^{(3)}$ that is absent from $\G_f$. Hence $\G_{\T^{\mathrm{SHAP}}}$ and $\G_f$ are not Markov-equivalent: they disagree on the d-separation $X_1^{(1)} \indep Y^{(3)} \mid X_1^{(2)}$, which holds in $\G_f$ but fails in $\G_{\T^{\mathrm{SHAP}}}$. The same argument applies to TimeSHAP~\cite{bento2021timeshap}, which inherits the coalition structure of SHAP on sequence models.
\end{proof}

The same argument applies to gradient saliency (which computes marginal partial derivatives, collapsing mediated paths) and attention weights (which are not formally tied to conditional independence~\cite{jain2019attention,wiegreffe2019attention}).

\section{Discussion and Open Challenges}
\label{sec:discussion}

\paragraph{Instantiating the transition tensor.}
Definition~\ref{def:transition-tensor} specifies the structural container; Definition~\ref{def:admissible} specifies the two conditions any DAG-faithful instantiation must satisfy. Demonstrating a concrete admissible instantiation that achieves DAG-faithfulness in general remains the central open problem. The failure modes (C1) and (C2) identified rule out naive Jacobian-based instantiations (which violate on-manifold evaluation) and marginal perturbation-based instantiations (which violate conditional separation). Candidate directions include local conditional-independence estimators on the model's predictive distribution and low-rank factorizations that exploit temporal structure to reduce the $O(D^2 T)$ sample complexity of the full tensor.

\paragraph{Joint Window Forecasting}
One of the main pain points in XAI for time series forecasting, and an issue we stated but did not address in this text, was the potential joint predictive behavior of $W>1$ forecasting models. Indeed, when discussing the model dependency graph, we focused on singular past time steps to forecast step edges and omitted intra-forecast-window dependency edges, a point that is also very valuable in the context of temporal-structure dependency-based XAI methods. Possible future directions include the addition of causal discovery methods for edge recovery in the forecast window, or methods leveraging the idea of Markov Chains, to understand joint predictive behavior.

\paragraph{Robustness to distributional shift.}
DAG-faithfulness is defined relative to a reference distribution $p$. Whether the induced graph $\G_{\T}$ is invariant under reasonable variations in $p$ is an open question and a natural target for theoretical analysis.

\paragraph{Identifiability.}
We require Markov equivalence rather than exact DAG recovery, since the latter is not identifiable from observational data in general. Strengthening exact recovery requires additional assumptions (e.g.\ linear-Gaussian models, or access to interventional data), which are worth exploring as restricted settings with sharper guarantees.

\paragraph{Human interpretability.}
A $D \times D \times T$ tensor is not directly human-readable. Visualization strategies (lag-aggregated heatmaps, interactive graph browsers, dimensionality-reduced summaries) are needed to make the structure accessible. We view this as a complementary challenge to the theoretical framework: the tensor is the right internal representation for faithfulness, but user-facing explanations will require principled projections of it.

\section{Conclusion}

We have argued that faithful global explanations for time series models cannot be achieved using conventional flat attribution scoring methods. We have defined \emph{DAG-faithfulness} as the appropriate criterion for global time series explanations, formalized the admissibility conditions that candidate instantiations must satisfy, shown that standard methods and their time-series-aware extensions fail these conditions by construction, and connected the framework to Granger causality and dynamic Bayesian networks. We call on the XAI community to adopt a structure-centric faithfulness criterion and the transition tensor representation as the foundation for explainability in temporal domains, and we identify the construction of admissible, DAG-faithful instantiations as the central open problem for future work.


\bibliographystyle{unsrtnat}
\bibliography{references}
\newpage
\mbox{}
\appendix
\section{Appendix}

\begin{proposition}[Computational Insufficiency of Flat Attribution Methods]
\label{prop:insufficiency}
Flat attribution methods cannot satisfy both admissibility conditions (C1) and (C2) in general. Specifically:
\begin{enumerate}
    \item Any flat attribution method whose global score reduces to marginal conditioning over a single input (i.e., conditioning only on $X^{(t_1)}$ and not on the remaining inputs) violates (C2): there exist models $f_1, f_2$ with non-isomorphic dependency graphs $\G_{f_1} \not\cong \G_{f_2}$ that the method assigns identical global attributions so that no threshold can separate them.
    \item Any gradient-based method that evaluates $\partial f/\partial x$ as a function on $\R^{T \times D}$ violates (C1): there exist models $f_1, f_3$ with $\G_{f_1} = \G_{f_3}$ that the method assigns different attributions, since the gradient registers off-manifold sensitivity that is not part of the model's on-manifold dependency structure.
\end{enumerate}
A DAG-faithful explanation, therefore, requires an attribution functional satisfying both (C1) and (C2), properties that no existing flat attribution scoring method can assure. The transition tensor (Definition~\ref{def:transition-tensor}) provides the minimal representational container in which both conditions can be expressed.
\end{proposition}

\begin{proof}
We prove the two parts via independent constructions.

\paragraph{Part 1: violation of (C2) by marginal-conditioning methods.}

Assume Scenario I

\noindent The expected marginal influence
\begin{equation}
    \Phi^f_{t_1}
    = \E_{p}\Bigl[
        \bigl|\,\E_p\bigl[f(\X) \mid X^{(t_1)}\bigr]
        - \E_p\bigl[f(\X)\bigr]\,\bigr|
    \Bigr]
    \label{eq:marginal}
\end{equation}
is the natural global characterization of any method that conditions only on $X^{(t_1)}$. Computing $\Phi^{f_1}_{1}$ via \eqref{eq:dgp}:
\begin{align*}
    &\E_p\bigl[f_1(\X) \mid X^{(1)} = x\bigr] - \E_p\bigl[f_1(\X)\bigr] \\
    &\quad= \beta\,\E_p\bigl[X^{(2)} \mid X^{(1)} = x\bigr] - \beta\,\E_p\bigl[X^{(2)}\bigr] \\
    &\quad= \beta\,(\delta x + \E_p[\varepsilon]) - \beta\,\E_p\bigl[\delta X^{(1)} + \varepsilon\bigr] \\
    &\quad= \beta\delta\,x - \beta\delta\,\E_p\bigl[X^{(1)}\bigr] \\
    &\quad= \beta\delta\,\bigl(x - \E_p[X^{(1)}]\bigr),
\end{align*}
where we used $\E_p[X^{(2)} \mid X^{(1)} = x] = \delta x$ from \eqref{eq:dgp} and $\E[\varepsilon] = 0$. Taking the expectation over $X^{(1)}$:
\begin{align*}
    \Phi^{f_1}_{1}
    &= \E_p\Bigl[|\beta\delta|\,\bigl|X^{(1)} - \E_p[X^{(1)}]\bigr|\Bigr] \\
    &= |\beta\delta|\,\mathrm{MAD}(X^{(1)}) > 0,
\end{align*}
where $\mathrm{MAD}$ denotes the mean absolute deviation, which is strictly positive since $\mathrm{Var}(X^{(1)}) > 0$ by assumption, and $\beta\delta \neq 0$. $\Phi^{f_1}_{1} > 0$ even though the edge $X^{(1)} \to X^{(3)}$ is absent from $\G_{f_1}$. The spurious attribution arises from temporal autocorrelation: marginal conditioning cannot separate the indirect path $X^{(1)} \to X^{(2)} \to f_1$ from a direct effect, which is exactly a violation of (C2). Similarly, computing $\Phi^{f_2}_{1}$:
\[
    \Phi^{f_2}_{1} = |\beta\delta + \gamma|\,\mathrm{MAD}(X^{(1)}).
\]
Setting $\gamma = -2\beta\delta$ yields $\Phi^{f_1}_{1} = \Phi^{f_2}_{1}$. The collapse propagates to any method whose global score is bounded below by a positive multiple of $\Phi^f_{t_1}$, including SHAP, LIME, and the temporally-aware perturbation methods (TimeSHAP, ShapTime) that condition perturbations on coalitions or windows of features rather than on the conditional distribution given the mediator.

\paragraph{Part 2: violation of (C1) by gradient methods.}

\begin{remark}[Gradient-based edge detection]
\label{rem:gradient-edge}
For a differentiable $f$, a nonzero expected absolute gradient
\[
\E_{p(\mathbf{x})}\!\left[\left| \frac{\partial f(\mathbf{x})_j}{\partial x_i^{(t_1)}} \right|\right] > 0
\]
is \emph{sufficient} but not necessary for the edge $X_i^{(t_1)} \to X_j^{(T+1)}$ to be present in $\G_f$. Sufficiency follows because a nonzero expected absolute gradient precludes conditional independence under $p(\mathbf{x})$. The converse fails in general: conditional dependence may arise through higher-order interactions not captured at first order, and gradients evaluated off the support of $p(\mathbf{x})$ may register sensitivity that does not correspond to on-manifold dependence (cf.\ the reparametrization construction in the proof of Proposition~\ref{prop:insufficiency}).
\end{remark}

Gradient-based methods probe $f$ via partial derivatives at points in input space. We use a separate construction in which the data manifold is low-dimensional, so off-manifold directions are unambiguous. 

Assume Scenario II,

The partial derivative
\[
\frac{\partial f_3}{\partial x^{(1)}} = -\gamma\delta \neq 0
\]
is nonzero everywhere, including at every $\mathbf{x} \in \supp(p)$. Gradient-based methods such as vanilla saliency, integrated gradients, and SmoothGrad assign nonzero attribution to $X^{(1)}$ under $f_3$ despite the absence of an edge. This is a violation of (C1): the method's score depends on $f$'s extension off $\supp(p)$, registering sensitivity in directions orthogonal to the manifold that play no role in the model's on-manifold dependency structure.

\end{proof}







\end{document}